\documentclass{article}
\usepackage{ijcai19}

\usepackage{times}
\usepackage{soul}
\usepackage{url}
\usepackage[hidelinks]{hyperref}

\usepackage[utf8]{inputenc}
\usepackage[small,skip=2pt]{caption}
\usepackage{graphicx}
\usepackage{amsmath}
\makeatletter
\def\maketag@@@#1{\hbox{\m@th\normalfont\normalsize#1}}
\makeatother
\usepackage{amsthm}
\usepackage{amsfonts}
\usepackage{dsfont}
\usepackage{booktabs}
\usepackage{algorithm}
\usepackage{algorithmic}
\usepackage{tabularx}
\usepackage{microtype}
\usepackage{stfloats}

\title{PiNet: A Permutation Invariant Graph Neural Network for Graph Classification}

\author{
  Peter Meltzer$^{1,2}$
\and
Marcelo Daniel Gutierrez Mallea$^2$\And
Peter J. Bentley$^{1,2}$
\affiliations
$^1$Department of Computer Science, University College London, UK\\
$^2$Braintree Ltd., London, UK\\
\emails
\{p.meltzer, p.bentley\}@cs.ucl.ac.uk,
\{p.meltzer, p.bentley, m.gutierrez\}@braintree.com
}

\newtheorem{theorem}{Theorem}[section]

\newtheorem{lemma}[theorem]{Lemma}

\renewcommand{\vec}[1]{\mathbf{#1}}
\renewcommand{\o}{\cdot}
\renewcommand{\P}{\mathbf{P}}

\newcommand{\T}{\top}
\theoremstyle{definition}
\newtheorem{definition}{Definition}[section]

\newcommand{\Real}[1]{\mathbb{R}^{#1}}

\begin{document}

\maketitle

\begin{abstract}
We propose an end-to-end deep learning learning model for graph classification and representation learning that is invariant to permutation of the nodes of the input graphs. We address the challenge of learning a fixed size graph representation for graphs of varying dimensions through a differentiable node attention pooling mechanism. In addition to a theoretical proof of its invariance to permutation, we provide empirical evidence demonstrating the statistically significant gain in accuracy when faced with an isomorphic graph classification task given only a small number of training examples.
  
We analyse the effect of four different matrices to facilitate the local message passing mechanism by which graph convolutions are performed vs. a matrix parametrised by a learned parameter pair able to transition smoothly between the former. Finally, we show that our model achieves competitive classification performance with existing techniques on a set of molecule datasets.
\end{abstract}

\section{Introduction}

Graph classification, the problem of predicting a label to each graph in a given set, is of significant interest in the bio- and chemo-informatics domains, among others; with typical applications in predicting chemical properties \cite{Li2016}, drug effectiveness \cite{Bunke2009}, protein functions \cite{Shervashidze2009}, and classification of segmented images \cite{Scarselli2009}.

There are two major challenges faced by graph classifiers: First, a problem of ordering, i.e. the ability to recognise isomorphic graphs as equivalent when the order of their nodes/edges are permuted, and second, in how to handle instances of varying dimensions, i.e. graphs with different numbers of nodes/edges. In image classification, the ordering of pixels is given, and instances differing in size may be scaled; however, for graphs the ordering of nodes/edges is typically arbitrary, and finding the analogous transformation to scaling an image is evidently non-trivial.

Typical approaches to solving these challenges include kernel methods, in which implicit kernel spaces circumvent the need to map each instance to a fixed size, ordered representation for classification \cite{Zhang2018}, and deep learning architectures with some explicit feature extraction method whereby a fixed size representation is constructed for each graph and passed to a CNN (or similar) classification model \cite{Niepert2016}. While the deep learning approaches often outperform the kernel methods with respect to scalability in the number of graphs, they require suitable fixed size representation for each graph, and typically cannot guarantee that isomorphic graphs will be interpreted as the same.

In order to address these issues, we propose PiNet; an end-to-end deep learning graph convolution architecture with guaranteed invariance to permutation of nodes in the input graphs. To present our model, we first review relevant literature, then present the architecture with proof of its invariance to permutation. We conduct three experiment to evaluate PiNet's effectiveness: We verify the utility in its invariance to permutation with a graph isomorphism classification task, we then test its ability to learn appropriate message passing matrices, and we perform a benchmark test against existing classifiers on a set of standard molecule classification datasets. Finally, we draw our conclusions and suggest further work.

\section{Background}

\subsection{Graph Kernels}


A graph kernel $\phi$ is a positive semi-definite function that maps graphs belonging to a space $\mathcal{G}$ to an inner product in some Hilbert space $\mathcal{H}$, $\phi: \mathcal{G} \rightarrow \mathcal{H}$. Graph classification can be performed in the mapped space $\mathcal{H}$ with standard classification algorithms, or using the \emph{kernel trick} (with SVMs, for example) the mapped feature space may be exploited implicitly. In this sense, kernel methods are well suited to deal with the high and variable dimensions of graph data, where explicit computation of such a feature space may not be possible.

Despite the large number of graph kernels found in the literature, they typically fall into just three distinct classes \cite{ShervashidzeNINOSHERVASHIDZE2011}: Graph kernels based on random walks and paths \cite{Borgwardt2005}, graph kernels based on frequencies of limited size subgraphs or graphlets \cite{Shervashidze2009}, and graph kernels based on subtree patterns where a similarity matrix between two graphs is defined by the number of matching subtrees in each graph \cite{Harchaoui2007}.


Although kernels are well suited to varying dimensions of graphs, their scalability is limited. In many cases they scale poorly to large graphs \cite{Shervashidze2010} and given their reliance on SVMs or full computation of a kernel matrix, they become intractable for large numbers of graphs.

\subsection{Graph Neural Networks}

Convolutional and recurrent neural networks, while successful in many domains, struggle with a graph inputs because of the arbitrary order in which the nodes of each instance may appear \cite{Ying2018b}. For node level tasks (i.e. node classification and link prediction) graph neural networks (GNNs) \cite{Scarselli2009} handle this issue well by integrating the neural network structure with that of the graph. In each layer, state associated with each node is propagated to its neighbours via a learned filter and then a non linear function is applied. Thus the network's inner layers learn a latent representation for each node. For example, the GCN \cite{Kipf2016} uses a normalised version of the graph adjacency matrix to propagate node features while learning spectral filters. The GCN model has received significant attention in recent years with several extensions already in the literature \cite{Hamilton2017b,Atwood2016,Romero2018}.

However, for graph level tasks the GNN model and its variants do not handle permutations of the nodes well. For example, for a pair of isomorphic graphs, corresponding nodes would receive corresponding outputs, but for the graph as whole, the node level outputs will not necessarily be given in the same order, thus two graphs may be given shuffled representations presenting an obvious challenge for a downstream classifier. One approach to solving this problem is proposed by \cite{Verma2018} where a permutation invariant layer based on computing the covariance of the data is added to the GCN architecture; however, computation of the covariance matrix $O(n^3)$ in the number of nodes, thus not an attractive solution for large graphs.

\subsection{Mixed Models}

Combining elements of kernel methods and neural networks, mixed models use an explicit method in order to generate vectorized representations of the graphs which are then passed to a conventional neural network. One benefit to this approach being that explicit kernels are typically more efficient when dealing with large numbers of graphs \cite{Kriege2015}.

For example, PATCHY-SAN \cite{Niepert2016} extracts fixed size localized patches by applying a graph labelling procedure given by the WL algorithm \cite{Weisfeiler1968} and a canonical labeling procedure from \cite{Mckay2013} to order the nodes. It then uses these patches to form a 3-dimensional tensor for each graph that is passed to a standard CNN for classification.

A similar procedure is presented in \cite{Daniel2019} where, in order to overcome the potential loss of information associated with the CNN convolution operation, a Capsule network is used to perform the classification.

\section{PiNet}

Formally, we consider the problem of predicting a label $\vec{y}$ for an unseen test graph, given a training set $\mathcal{G}$ with corresponding labels $\mathcal{Y}_L$. Each graph $G \in \mathcal{G}$ is defined as the pair $G = (\vec{A}, \vec{X})$, where $\vec{A} \in \Real{N \times N}$ is the graph adjacency matrix, and $\vec{X} \in \Real{N \times d}$ is a corresponding matrix of $d$-dimensional node features. We fix $N$ to be the maximum number of nodes in each of the graphs in $\mathcal{G}$, padding empty rows and columns of $\vec{A}$ and $\vec{X}$ with zeros. Note, however, that these zero entries do not form part of the final graph representations used by the model.

\subsection{Model Architecture}

\begin{figure}[htb]
  \centering
  \includegraphics[width=\linewidth]{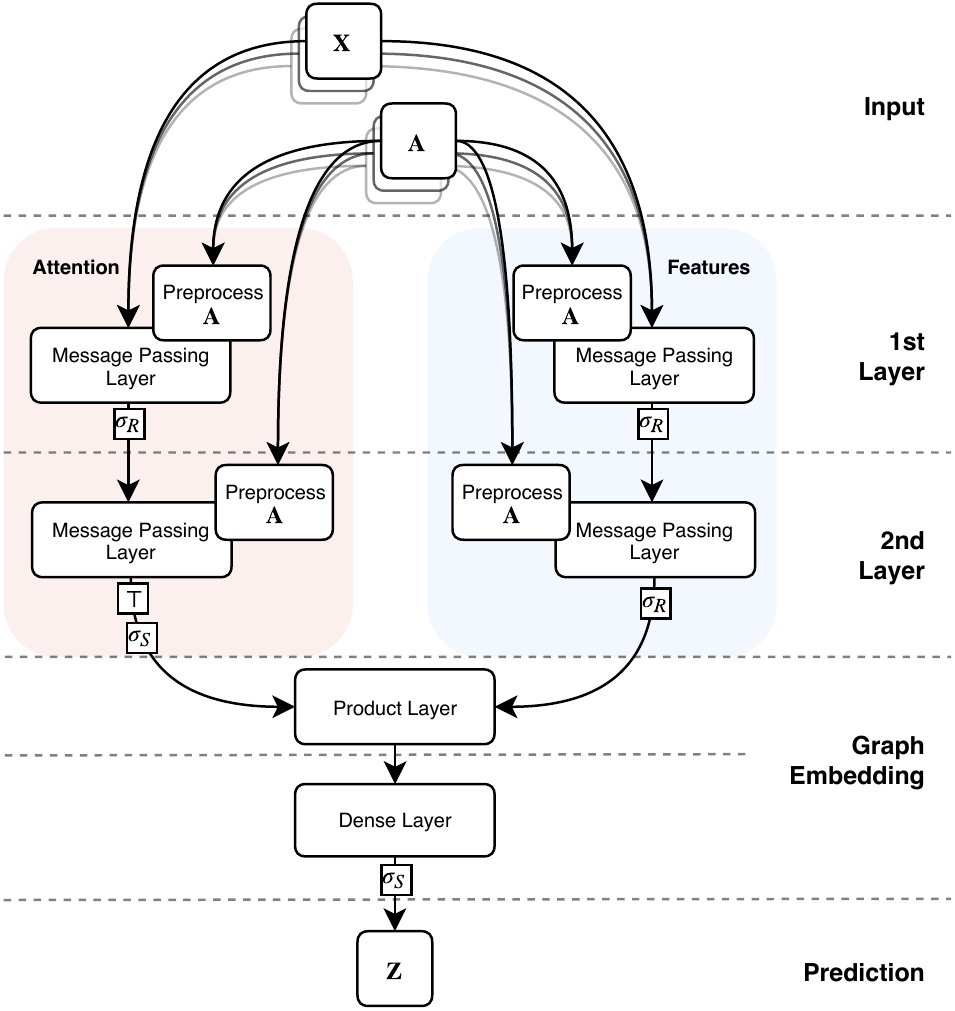}
  \caption{Model Architecture. $\vec{A}$ is the  adjacency matrix of a single graph, $\vec{X}$ the corresponding feature matrix, and $\vec{Z}$ the predicted label(s) for the complete batch of input graphs.}
  \label{fig:architecture}
\end{figure}

PiNet is an end-to-end deep neural network architecture that utilizes the permutation equivariance of graph convolutions in order to learn graph representations that are invariant to permutation.

As shown in \autoref{fig:architecture}, our model consists of a pair of double-stacked message passing layers combined by a matrix product. Let $\sigma_R$ and $\sigma_S$ denote the rectified linear unit and softmax activations functions, $\vec{\tilde{A}}$ the preprocessed adjacency matrix, and \small$F_X^{(1)}$\normalsize and \small$F_A^{(1)}$\normalsize the dimensions per node of the latent representation of the features and attention stacks respectively. The features and attention stacks each output a tensor ($Z_X$ and $Z_A$) with each element corresponding to an input graph given by the functions \small$z_X : (\Real{N \times N}, \Real{N \times d}) \rightarrow \Real{N \times F_X^{(1)}}$\normalsize and\small $z_A : (\Real{N \times N}, \Real{N \times d}) \rightarrow \Real{F_A^{(1)} \times N}$\normalsize respectively, where
\begin{equation}
  z_X(\vec{A}, \vec{X}) = \sigma_R \left(\vec{\tilde{A}} \; \sigma_R \left( \vec{\tilde{A}} \; \vec{X} \vec{W}_X^{(0)} \right) \vec{W}_X^{(1)} \right),
  \label{eq:zx}
\end{equation}
\begin{equation}
  z_A(\vec{A}, \vec{X}) = \sigma_S \left( \left[ \vec{\tilde{A}} \; \sigma_R \left( \vec{\tilde{A}} \; \vec{X} \vec{W}_A^{(0)} \right) \vec{W}_A^{(1)} \right]^\T \right),
  \label{eq:za}
\end{equation}
and
\begin{equation}
  \vec{\tilde{A}} = (p \vec{I} + (1 - p) \vec{D})^{-\frac{1}{2}} (\vec{A} + q \vec{I}) (p \vec{I} + (1 - p) \vec{D})^{-\frac{1}{2}},
  \label{eq:messages}
\end{equation}
where $\vec{I}$ is the identity matrix of order $N$, $\vec{D}$ is the degree matrix such that
\begin{equation}
  \label{eq:deg}
  \vec{D}_{ij} = d(\vec{A})_{ij} =
\left\{
	\begin{array}{ll}
	  \sum_{k=1}^N \vec{A}_{ik}  & \mbox{if } i = j, \\
	  0 & \mbox{otherwise,}
	\end{array}
\right.
\end{equation}
and $0 \le p, q \le 1$.

The use of softmax activation on the attention stack applies the constraint that outputs sum to $1$, thus preventing all node attention weightings from dropping to $0$ and keeping the resulting products within a reasonable range.

The trainable parameters $p$ and $q$ offer an extra attention mechanism that enables the model to weigh the importance of symmetric normalisation of the adjacency matrix, and the addition of self loops. \autoref{tbl:matrices} shows the four matrices given by the extreme cases of $p$ and $q$; however, intermediate combinations are of course possible. Also note that $p$ and $q$ may be different for each message passing layer.

\begin{table}[h!]
  \small
  \begin{tabularx}{\linewidth}{lXcc}
    \toprule
    Matrix & Definition & $p$ & $q$ \\
    \midrule
    Adjacency & $\vec{A}$ & $1$ & $0$ \\
    $\vec{A}$ with S.L. & $\vec{A} + \vec{I}$ & $1$ & $1$ \\
    Sym Norm Adjacency & $\vec{D}^{-\frac{1}{2}} \vec{A} \vec{D}^{-\frac{1}{2}}$ & $0$ & $0$ \\
    Sym Norm $\vec{A}$ with S.L. & $\vec{D}^{-\frac{1}{2}} (\vec{A} + \vec{I}) \vec{D}^{-\frac{1}{2}}$ & $0$ & $1$ \\
    \bottomrule
  \end{tabularx}
  \caption{The four message passing matrices given by the extreme values of $p$ and $q$, i.e. $(p, q) \in \{0, 1\} \times \{0, 1\}$.}
  \label{tbl:matrices}
\end{table}

As seen in \autoref{fig:prop}, the weighting on the self loops given by $q$ allows the model to include each node's own state in the previous layer as an input to the next layer.

\begin{figure}[htb]
  \centering
  \includegraphics[width=\linewidth]{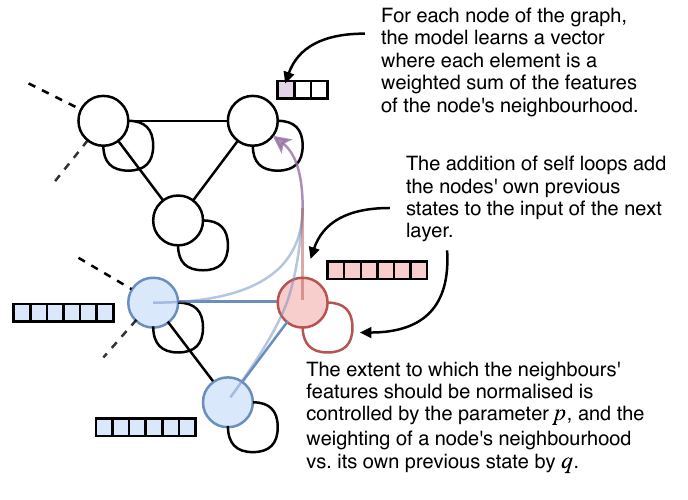}
  \caption{Layer propagation}
  \label{fig:prop}
\end{figure}

The final output of the model is the matrix $\vec{Z}$, where each row $i$ is the predicted label given by the function
\small
\begin{equation}
  \label{eq:z}
  z(\vec{A}, \vec{X}) = \sigma_S \left[ g\left( z_A\left(\vec{A}, \vec{X}\right) \cdot z_X\left(\vec{A}, \vec{X}\right) \right) \vec{W}_D \right] \in \Real{C},
\end{equation}
\normalsize
where \small$g : \Real{F_A^{(1)} \times F_X^{(1)}} \rightarrow \Real{F_A^{(1)} \cdot F_X^{(1)}}$\normalsize is a reshape function. Note that since $Z_A$ has been transposed, the columns of each matrix (corresponding directly to the nodes of the graph) align exactly with the rows of each matrix in $Z_X$, thus $\forall i \in |\mathcal{G}|$ the product \small$\left[Z_A\right]_i \cdot \left[Z_X\right]_i \in \Real{F_A^{(1)} \times F_X^{(1)}}$\normalsize aligns the weights for each node learned by the attention stack with the latent features learned by the features stack, resulting in a weighted sum of the features of each node without being affected by permutations of the nodes at input.

Finally, to train the model, we minimise the categorical cross-entropy
\small
\begin{equation}
  \mathcal{L} = - \sum_{l \in \mathcal{Y}_L} \sum_{f=1}^C \vec{Y}_{lf} \ln \vec{Z}_{lf},
\end{equation}
\normalsize
where $\vec{Y}$ is the target labels, and $C$ the number of classes.

\subsection{Permutation Invariance}

Here, following the necessary definitions, we provide proof of our model's invariance to permutation of the nodes in the input graphs.

\begin{definition}[Permutation]
  A permutation $\pi$ is a bijection of a set $S$ onto itself, such that it moves an object at position $i$ to position $\pi(i)$. For example, a permutation of $\pi$ to the vector $(x_1, x_2, \dots, x_n)$ would be $(x_{\pi^{-1}(1)}, x_{\pi^{-1}(2)}, \dots, x_{\pi^{-1}(n)})$.
\end{definition}

\begin{definition}[Permutation Matrix]
  A permutation matrix $\P_\pi \in \{0, 1\}^{n \times n}$ is an orthogonal matrix such that:
  \begin{align}
    \left[ \P_\pi \vec{X} \right]_{ij} = \vec{X}_{\pi^{-1}(i)j}, \qquad \vec{X} \in \Real{n \times m} \\
    \left[ \vec{X} \P_\pi^\T \right]_{ij} = \vec{X}_{i\pi^{-1}(j)}, \qquad \vec{X} \in \Real{m \times n}
  \end{align}
  thus, for a square matrix $\vec{A} \in \Real{n \times n}$,
  \begin{equation}
    \left[ \P_\pi \vec{A} \P_\pi^\T \right]_{ij} = \vec{A}_{\pi^{-1}(i)\pi^{-1}(j)}.
  \end{equation}
\end{definition}

\begin{definition}[Graph Permutation]
  We define the permutation of $\pi$ to a graph $G$ as a mapping of the node indices, i.e.
  \begin{equation}
    \P_\pi G = (\P_\pi \vec{A} \P_\pi^\T, \P_\pi \vec{X})
  \end{equation}
  \label{def:graphperm}
\end{definition}

\begin{definition}[Permutation Invariance]
  Let $\mathcal{P}_n$ be the set of all valid permutation matrices of order $n$, then a function $f$ is invariant to row permutation iff
  \begin{equation}
    f(\vec{X}) = f(\P_\pi \vec{X}), \qquad \forall \vec{X} \in \Real{n \times m}, \P_\pi \in \mathcal{P}_n,
  \end{equation}
  and $f$ is invariant to column permutation iff
  \begin{equation}
    f(\vec{X}) = f(\vec{X} \P_\pi^\T), \qquad \forall \vec{X} \in \Real{m \times n}, \P_\pi \in \mathcal{P}_n.
  \end{equation}
\end{definition}

\begin{definition}[Permutation Equivariance]
  Let $\mathcal{P}_n$ be the set of all valid permutation matrices of order $n$, then a function $f$ is \emph{equivariant} to row permutation iff
  \begin{equation}
    \P_\pi f(\vec{X}) = f(\P_\pi \vec{X}), \qquad \forall \vec{X} \in \Real{n \times m}, \P_\pi \in \mathcal{P}_n,
  \end{equation}
  similarly, for column permutation $f(\vec{X}) \P_\pi^\T = f(\vec{X} \P_\pi^\T)$.
\end{definition}

\begin{lemma}
  \label{lma:dot}
  For any matrices $\vec{A} \in \Real{n \times m}$ and $\vec{B} \in \Real{n \times p}$, and any permutation $\pi$, the product $\vec{A}^\T \cdot \vec{B}$ remains unchanged by a permutation of $\pi$ applied to the rows of $\vec{A}$ and $\vec{B}$, i.e.
  \begin{equation}
    (\P_\pi \vec{A})^\top \o \P_\pi \vec{B} = \vec{A}^\top \o \vec{B}.
  \end{equation}
\end{lemma}

\newcommand{\pin}[1]{{\pi^{-1}{(#1)}}}

\begin{proof}
    Consider the vector product
    \begin{equation}
    \mathbf{a}^\T \cdot \mathbf{b} = \sum_k^n \vec{a}_k \vec{b}_k.
  \end{equation}
    We permute the rows of $\vec{a}$ and the columns of $\vec{b}$
    \begin{equation}
    \P_\pi \vec{a}^\T \cdot \vec{b} \P_\pi^\T = \sum_k^n \vec{a}_\pin{k} \vec{b}_\pin{k} \\
  \end{equation}
   and observe a reordering of terms, in which the factor pairs remain in correspondence. Since addition is commutative, then
    \begin{equation}
    \sum_k^n \vec{a}_\pin{k} \vec{b}_\pin{k} = \sum_k^n \vec{a}_k \vec{b}_k \\
  \end{equation}
    \begin{equation}
    \implies \P_\pi \vec{a}^\T \cdot \vec{b} \P_\pi^\T = \vec{a}^\T \cdot \vec{b}. \\
  \end{equation}
  By the same logic, we see that
  \begin{align}
    \left[ (\P_\pi \vec{A})^\top \o \P_\pi \vec{B} \right]_{ij} &= \sum_k^n \vec{A}_{\pin{k}i} \vec{B}_{\pin{k}j} \\
    &= \sum_k^n \vec{A}_{ki} \vec{B}_{kj} \\
    &= \left[ \vec{A}^\top \o \vec{B} \right]_{ij} \\
    \implies (\P_\pi \vec{A})^\top \o \P_\pi \vec{B} &= \vec{A}^\top \o \vec{B}
  \end{align}
\end{proof}

\begin{lemma}
  \label{lma:pre}
  If $\vec{\tilde{A}}$ is the preprocessed version of $\vec{A}$, then $\P_\pi \vec{\tilde{A}} \P_\pi^\T$ is the preprocessed version of $\P_\pi \vec{A} \P_\pi^\T$, i.e. If
  \begin{equation}
    \vec{\tilde{A}} = (p \vec{I} + (1 - p) d(\vec{A}))^{-\frac{1}{2}} (\vec{A} + q \vec{I}) (p \vec{I} + (1 - p) d(\vec{A}))^{-\frac{1}{2}}
  \end{equation}
    then
  \begin{multline}
  \P_\pi \vec{\tilde{A}} \P_\pi^\T = (p \vec{I} + (1 - p) d(\vec{\P_\pi \vec{A} \P_\pi^\T}))^{-\frac{1}{2}} (\P_\pi \vec{A} \P_\pi^\T + q \vec{I}) \\ (p \vec{I} + (1 - p) d(\vec{\P_\pi \vec{A} \P_\pi^\T}))^{-\frac{1}{2}},
  \label{eq:pap}
  \end{multline}
  where $d(\vec{A})$ is the diagonal degree matrix of $\vec{A}$ as defined in \autoref{eq:deg}.
\end{lemma}

\begin{proof}
  From \autoref{eq:deg},
  \begin{align}
    d(\vec{\P_\pi A \P_\pi^\T})_{ij} &=
\left\{
	\begin{array}{ll}
	  \sum_{k=1}^N \vec{A}_{\pi^{-1}(i) \pi^{-1}(k)}  & \mbox{if } i = j, \\
	  0 & \mbox{otherwise}
	\end{array}
\right. \\
    &= \P_\pi d(\vec{A}) \P_\pi^\T
  \end{align}
  Considering each factor of \autoref{eq:pap} (RHS), we observe that
  \begin{multline}
  (p \vec{I} + (1 - p) d(\vec{\P_\pi \vec{A} \P_\pi^\T}))^{-\frac{1}{2}} \\
    = (p \vec{I} + (1 - p) \P_\pi d(\vec{\vec{A} }) \P_\pi^\T )^{-\frac{1}{2}}
  \end{multline}
  \begin{align}
    &= (p \vec{I} + \P_\pi \left[ (1 - p) d(\vec{\vec{A} }) \right] \P_\pi^\T )^{-\frac{1}{2}} \\
    &= (\P_\pi \left[ p \vec{I} + (1 - p) d(\vec{\vec{A} }) \right] \P_\pi^\T )^{-\frac{1}{2}}
    \intertext{and since the matrix is diagonal}
    &= \P_\pi ( p \vec{I} + (1 - p) d(\vec{\vec{A} }) )^{-\frac{1}{2}}\P_\pi^\T.
    \label{eq:firstfactor}
  \end{align}
  We also have
  \begin{equation}
    (\P_\pi \vec{A} \P_\pi^\T + q \vec{I}) = \P_\pi (\vec{A} + q \vec{I}) \P_\pi^\T.
    \label{eq:secondfactor}
  \end{equation}
  By \autoref{eq:firstfactor} and \ref{eq:secondfactor},  
  \small
  \begin{multline}
  (p \vec{I} + (1 - p) d(\vec{\P_\pi \vec{A} \P_\pi^\T}))^{-\frac{1}{2}} (\P_\pi \vec{A} \P_\pi^\T + q \vec{I}) \\
  \qquad \qquad \qquad \qquad (p \vec{I} + (1 - p) d(\vec{\P_\pi \vec{A} \P_\pi^\T}))^{-\frac{1}{2}} \\
  = \P_\pi (p \vec{I} + (1 - p) d(\vec{A}))^{-\frac{1}{2}} \P_\pi^\T \P_\pi (\vec{A} + q \vec{I}) \P_\pi^\T \qquad \qquad \\
  \qquad \qquad \qquad \qquad \qquad \P_\pi (p \vec{I} + (1 - p) d(\vec{A}))^{-\frac{1}{2}} \P_\pi^\T,
  \end{multline}
  and since $\P_\pi$ is orthogonal, $\P_\pi^\T \P_\pi = \vec{I}$, so
  \begin{multline}
  = \P_\pi (p \vec{I} + (1 - p) d(\vec{A}))^{-\frac{1}{2}} (\vec{A} + q \vec{I}) (p \vec{I} + (1 - p) d(\vec{A}))^{-\frac{1}{2}} \P_\pi^\T
  \end{multline}
  \normalsize
\end{proof}

\begin{theorem}
  For any input graph $G$, and any permutation $\pi$ applied to $G$, the output of PiNet is equal, i.e.
  \begin{equation}
    z(G) = z( \P_\pi G ),
  \end{equation}
  where $z : (\Real{N \times N}, \Real{N \times d}) \rightarrow \Real{C}$ is the forward pass function of PiNet given in \autoref{eq:z}.
\end{theorem}

\begin{proof}
  By \autoref{eq:zx}, Definition \ref{def:graphperm} and Lemma \ref{lma:pre},
  \begin{align}
    \label{eq:zxstart}
    z_X&(\P_\pi G) = z_X(\P_\pi \vec{A} \P_\pi^\T, \P_\pi \vec{X}) \\
    &= \mbox{\small $\sigma_R \left(\vec{\P_\pi \tilde{A} \P_\pi^\T} \; \sigma_R \left( \vec{\P_\pi \tilde{A} \P_\pi^\T} \;\P_\pi  \vec{X} \vec{W}_X^{(0)} \right) \vec{W}_X^{(1)} \right) $}
  \end{align}
$\P_\pi$ is orthogonal, so $\P_\pi^\T \P_\pi = \vec{I}$, giving
\begin{multline}
    z_X(\P_\pi G) \\
    = \sigma_R \left(\vec{\P_\pi \tilde{A} \P_\pi^\T} \; \sigma_R \left( \vec{\P_\pi \tilde{A} } \; \vec{X} \vec{W}_X^{(0)} \right) \vec{W}_X^{(1)} \right)
\end{multline}
Since $\sigma_R$ is an element-wise operation,
  \begin{equation}
    \sigma_R(\P_\pi \vec{X}) = \P_\pi \cdot \sigma_R(\vec{X}),
  \end{equation}
then
  \begin{multline}
    z_X(\P_\pi G) \\
    = \mbox{\small$ \sigma_R \left(\vec{\P_\pi \tilde{A} \P_\pi^\T} \; \P_\pi \cdot \sigma_R \left( \vec{\tilde{A} } \; \vec{X} \vec{W}_X^{(0)} \right) \vec{W}_X^{(1)} \right)$}
  \end{multline}
  \begin{align}
    \label{eq:zxlater}
    &= \sigma_R \left(\vec{\P_\pi \tilde{A}} \; \sigma_R \left( \vec{\tilde{A} } \; \vec{X} \vec{W}_X^{(0)} \right) \vec{W}_X^{(1)} \right) \\
    &= \P_\pi \cdot \sigma_R \left(\vec{\tilde{A}} \; \sigma_R \left( \vec{\tilde{A} } \; \vec{X} \vec{W}_X^{(0)} \right) \vec{W}_X^{(1)} \right) \\
    &= \P_\pi \cdot z_X(\vec{A}, \vec{X}) \\
    \label{eq:zxfinal}
    &= \P_\pi \cdot z_X(G).
  \end{align}
  By the same logic as \autoref{eq:zxstart} to \ref{eq:zxlater},
  \small
  \begin{multline}
    z_A(\P_\pi G) \\
    = \mbox{\small$ \sigma_S \left( \left[ \vec{\P_\pi \tilde{A} \P_\pi^\T} \; \sigma_R \left( \vec{\P_\pi \tilde{A} \P_\pi^\T} \;\P_\pi  \vec{X} \vec{W}_X^{(0)} \right) \vec{W}_X^{(1)} \right]^\T \right)$}
  \end{multline}
  \normalsize
  \begin{align}
    &= \sigma_S \left( \left[ \vec{\P_\pi \tilde{A}} \; \sigma_R \left( \vec{\tilde{A} } \; \vec{X} \vec{W}_X^{(0)} \right) \vec{W}_X^{(1)} \right]^\T \right) \\
    \label{eq:zamiddle}
    &= \left[ \sigma_{S'} \left( \vec{\P_\pi \tilde{A}} \; \sigma_R \left( \vec{\tilde{A} } \; \vec{X} \vec{W}_X^{(0)} \right) \vec{W}_X^{(1)} \right) \right]^\T,
  \end{align}
where $\sigma_{S'}$ is a column-wise softmax
  \begin{equation}
    \sigma_{S'}(\vec{X})_{ij} = \frac{e^{\vec{X}_{ij}}} {\sum_k e^{\vec{X}_{ik}}}.
  \end{equation}
When the rows of its input are permuted by $\pi$,
  \begin{equation}
    \sigma_{S'}(\P_\pi \vec{X})_{ij} = \frac{e^{\vec{X}_{\pi^{-1}(i)j}}} {\sum_k e^{\vec{X}_{\pi^{-1}(i)k}}}
    \label{eq:softmax}
  \end{equation}
  we observe that rows of the output are also permuted by $\pi$, thus by \autoref{eq:zamiddle} and \ref{eq:softmax}
  \small
  \begin{equation}
    z_A(\P_\pi G) = \left[ \P_\pi \cdot \sigma_{S'} \left( \vec{\tilde{A}} \; \sigma_R \left( \vec{\tilde{A} } \; \vec{X} \vec{W}_X^{(0)} \right) \vec{W}_X^{(1)} \right) \right]^\T.
    \label{eq:zafinal}
  \end{equation}\normalsize
From \autoref{eq:z}
\begin{equation}
  \label{eq:zp}
  z(\P_\pi G) = \sigma_S \left[ g\left( z_A \left( \P_\pi G\right) \cdot z_X\left(\P_\pi G\right) \right) \vec{W}_D \right],
\end{equation}
and by \autoref{eq:zxfinal} and \ref{eq:zafinal} we see that
\small
  \begin{multline}
    \label{eq:zazxstart}
    z_A \left( \P_\pi G\right) \cdot z_X\left(\P_\pi G\right) \\
    = \left[ \P_\pi \cdot \sigma_{S'} \left( \vec{\tilde{A}} \; \sigma_R \left( \vec{\tilde{A} } \; \vec{X} \vec{W}_X^{(0)} \right) \vec{W}_X^{(1)} \right) \right]^\T \\
    \cdot \left[ \P_\pi \cdot z_X(G) \right],
  \end{multline}
  \normalsize
which by Lemma \ref{lma:dot}
  \begin{align}
    &= \left[ \sigma_{S'} \left( \vec{\tilde{A}} \; \sigma_R \left( \vec{\tilde{A} } \; \vec{X} \vec{W}_X^{(0)} \right) \vec{W}_X^{(1)} \right) \right]^\T \cdot \left[ z_X(G) \right] \\
    &= \sigma_{S} \left( \left[ \vec{\tilde{A}} \; \sigma_R \left( \vec{\tilde{A} } \; \vec{X} \vec{W}_X^{(0)} \right) \vec{W}_X^{(1)} \right]^\T \right) \cdot \left[ z_X(G) \right] \\
    \label{eq:zazxend}
    &=  z_A(G) \cdot z_X(G)
  \end{align}
  Finally, by \autoref{eq:zp} and \ref{eq:zazxstart} to \ref{eq:zazxend},
  \begin{align}
    z(\P_\pi G) &= \sigma_S \left[ g\left( z_A \left( G\right) \cdot z_X\left(G\right) \right) \vec{W}_D \right] \\
  &= z(G).
  \end{align}
\end{proof}

\subsection{Implementation}

To implement PiNet we use Keras + Tensorflow. The model operates on batches and uses a mixture of Scipy sparse matrices and Numpy arrays to represent the graphs and their features. Full source code for PiNet is available at LINK(reveals authors).

\section{Experiments}

We conduct three experiments: We empirically verify the utility of our model's invariance to permutation with a graph isomorphism classification task, we evaluate the effect of different message passing matrices and the model's ability to select an appropriate message passing configuration, and we compare the model's classification performance against existing graph classifiers on a set of standard molecule classification datasets. We next describe the data used followed by a description of each experiment.

\subsection{Datasets}

For the isomorphism test, we generate a dataset of 500 graphs. To create sufficient challenge, in each randomly sampled Erdos Reny \cite{ErdOs1960} graph, we fix the number of nodes, and the node degree distributions, to be constant. We generate the dataset according to \autoref{alg:graphGen}, with the parameters: $N = 50$, $C = 5$, $N_g = 100$, and $p = 0.15$.

For the final two experiments we use the binary classification molecule datasets detailed in \autoref{tbl:graphstats}.

\begin{algorithm}[htb]
  \small
   \caption{Graph Dataset Generation}
   \label{alg:graphGen}
\begin{algorithmic}

\STATE {\bfseries Input:} Number of nodes $N$,
number of classes $C$,
number of graphs per class $N_{g}$,
edge probability $p$

\STATE seedGraph $\gets$ SampleErdosRenyiGraph($N$, $p$)
 \WHILE {seedGraph not fully connected  }
\STATE{ SampleErdosRenyiGraph($N$, $p$)}
 \ENDWHILE

\STATE $D$ $\gets$ Array[]

\FOR{ \textbf{each} $c$ {\bfseries in} $C$}

	\STATE $G_c$ $\gets$ Array[]

	\STATE S $\gets$ getDegreeSequence(seedGraph)
	\STATE sampleGraph  $\gets$ GenerateGraph(S)	
	
	\FOR{ each $i$  {\bfseries in} $N_{g}$}

		\STATE $p_g$ $\gets$ permute(sampleGraph.adjacencyMatrix)
		\STATE $p_g$ $\gets$ relabel(pg.nodes)
		
		\STATE $G_c[i]$ $\gets$ $p_g$
		
	\ENDFOR
	\STATE $D[c]$ $\gets$ $G_c$
\ENDFOR

\STATE {\bfseries Output:} D

\end{algorithmic}
\end{algorithm}

\begin{table}[b]
  \small
  \centering
  \begin{tabularx}{\linewidth}{lrrrrr}
  \toprule
 & MUTAG & NCI1 & NCI109 & PTC & PROTEINS \\ \midrule
 $|\mathcal{G}|$ & 188 & 4110 & 4127 & 344 & 1113 \\
 Max. $|V|$ & 28 & 111 & 111 & 109 & 620 \\
 Mean $|V|$ & 18 & 29.8 & 29.6 & 25.56 & 39.06 \\
 $d$ & 7 & 37 & 38 & 18 & 3 \\
 \% of +ve & 66.49 & 50.05 & 50.38 & 39.51 & 59.57 \\
\bottomrule
\end{tabularx}
\caption{Binary classification molecule datasets. $|\mathcal{G}|$ is the number of graphs, $|V|$ the number of nodes, and $d$ the dimensions of the node features.}
\label{tbl:graphstats}
\end{table}

\subsection{Isomorphism Test}
\label{sse:isomorphismexp}

We test PiNet's ability to recognise isomorphic graphs - specifically, unseen permutations of given training graphs. For a baseline, we test against two variants of the GCN \cite{Kipf2016}, one in which the graph level representation is given by a sum of the node representations, and the other in which we apply a dense layer directly to the node level outputs of the GCN. We also compare against two state of the art graph classifiers: the WL Kernel \cite{ShervashidzeNINOSHERVASHIDZE2011} and PATCHY-SAN \cite{Niepert2016}. We perform 10 trials for every training sample size.

\subsection{Message Passing Mechanisms}
\label{exp:message}

We study the impact on classification accuracy of the four matrices shown in \autoref{tbl:matrices} that facilitate message passing between nodes, alongside the parametrised matrix shown in \autoref{eq:messages}, in which the extent to which to normalise neighbours' values and include the node's own state as input are controlled by the learned parameters $p$ and $q$. Note that $p$ and $q$ are learned for each message passing layer and are not required to be the same, however, for each of the matrices from \autoref{tbl:matrices} we test, we use the same matrix in all four message passing layers.

For each matrix, we perform a 10-fold cross validation and record the classification accuracy. For all runs we use the following hyper-parameters: batch size = 50, epochs = 200, first layer latent feature size \small$F_X^{(0)}$\normalsize = \small$F_A^{(0)}$\normalsize = 100, second layer latent feature size \small$F_X^{(1)}$\normalsize = \small$F_A^{(1)}$\normalsize = 64, and learning rate = $10^{-3}$.

\subsection{Comparison Against Existing Methods}

\begin{table*}[b]
  \small
  \centering
\begin{tabular}{lccccc}
\toprule
 &            MUTAG &            NCI-1 &          NCI-109 &         PROTEINS &              PTC \\
\midrule
GCN + Dense &  $.86 \pm .06$ &  $.73 \pm .03$ &  $.72 \pm .02$ &  $.71 \pm .04$ &  $.63 \pm .07$ \\
GCN + Sum &  $.86 \pm .05$ &  $.72 \pm .03$ &  $.73 \pm .03$ &  $.74 \pm .04$ &  $.61 \pm .05$ \\
PATCHY-SAN & $.85 \pm .06$ & $.58 \pm .02$ & $.58 \pm .03$ & $.70 \pm .02$ & $.58 \pm .02$ \\
WLKernel   &  $.68 \pm .00^{*}$ &  $.53 \pm .02^{*}$ &  $.53 \pm .03^{*}$ &  $.61 \pm .01^{*}$ &  $.62 \pm .03$ \\
\midrule
PiNet (Manual $p$ and $q$) &  $.87 \pm .08$ &  $.74 \pm .03$ &  $.73 \pm .03$ &  $.75 \pm .06$ &  $.63 \pm .06$ \\
PiNet (Learned $p$ and $q$) &  $.88 \pm .07$ &  $.74 \pm .02$ &  $.71 \pm .04$ &  $.75 \pm .06$ &  $.63 \pm .04$ \\
\bottomrule
\end{tabular}
\caption{Mean classification accuracies for each classifier. For manual search the values $p$ and $q$ as follows: MUTAG and PROTEINS $p = 1, q = 0$, NCI-1 and NCI-109 $p = q = 1$, PTC $p = q = 0$. $^{*}$ indicates PiNet (both models) achieved statistically significant gain.}
\label{tbl:benchresults}
\end{table*}

As with the isomorphism test, we compare against the GCN with a dense layer applied directly to the node outputs as well as with a sum of the node representations, the Weisfeiler Lehman graph kernel, and PATCHY-SAN. For each model we perform 10-fold cross validation on each dataset.

For PiNet, we use the same hyper-parameters as described in Section \ref{exp:message}. We test the model with $p$ and $q$ as trainable values, and also search over the space $(p, q) \in \{0, 1\} \times \{0, 1\}$. For the GCN we use two layers of sizes $100$ and $64$ hidden units, with a learning rate of $10^{-3}$, and for PATCHY-SAN we search over two labelling procedures, betweenness centrality \cite{Brandes2001} and NAUTY \cite{Mckay2013} canonical labelling.

\section{Results}

\subsection{Isomorphism Test}

\begin{figure}[ht]
  \centering
  \includegraphics[width=\linewidth]{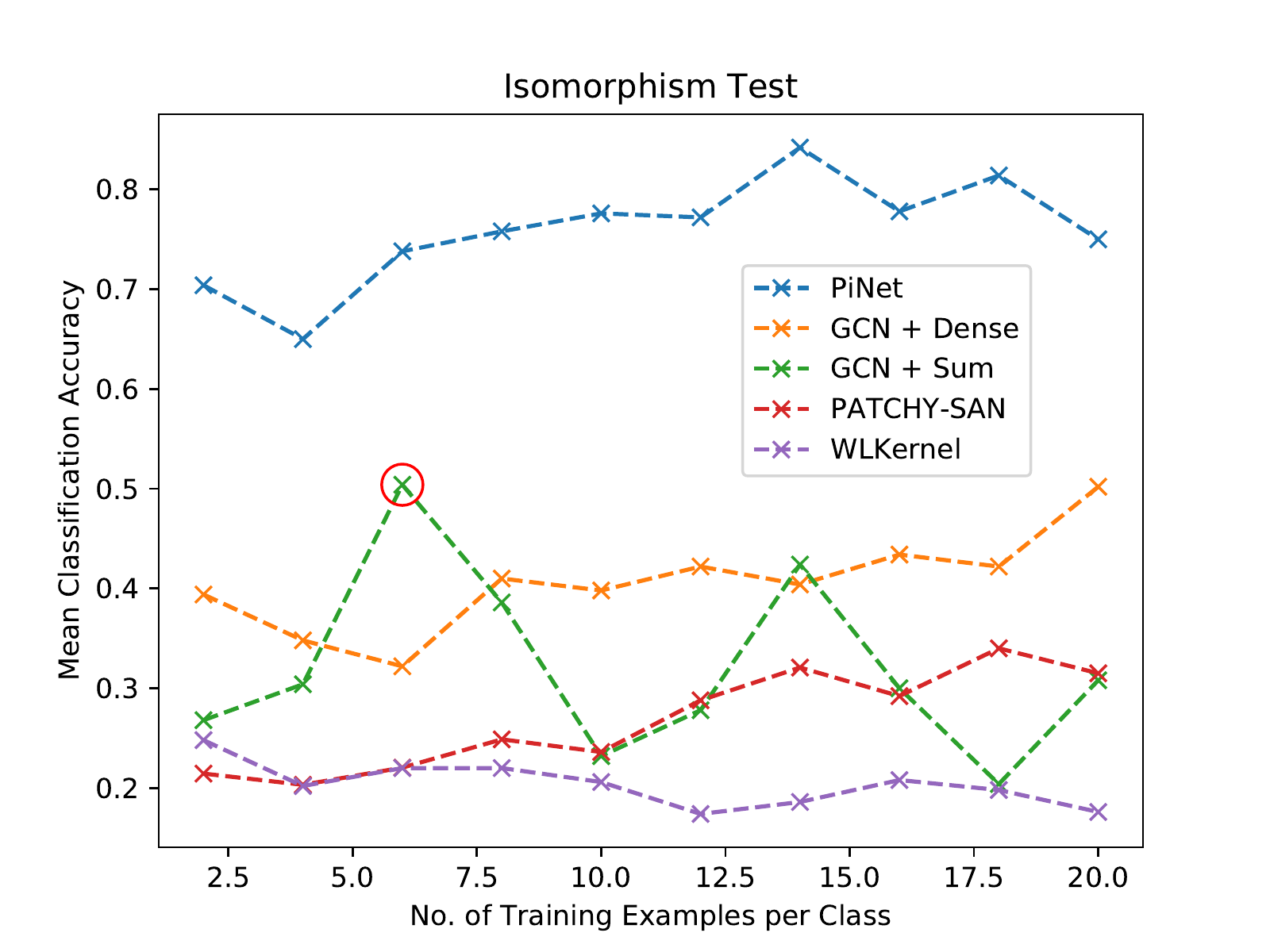}
  \caption{Mean classification accuracy on isomorphic graph classes.}
  \label{fig:isomorphic}
\end{figure}

As seen in \autoref{fig:isomorphic}, PiNet outperforms all competitors tested. Using an independent two-sample $t$-test we observe statistical significance ($p$-value $< 0.05$) in all cases except a single point (circled in red). PiNet fails to achieve 100\% accuracy since the neural network learns a surjective function, thus with so few training examples in some cases complete multiple classes become indistinguishable. 

\subsection{Message Passing Mechanisms}

\begin{figure}[ht]
  \centering
  \includegraphics[width=\linewidth]{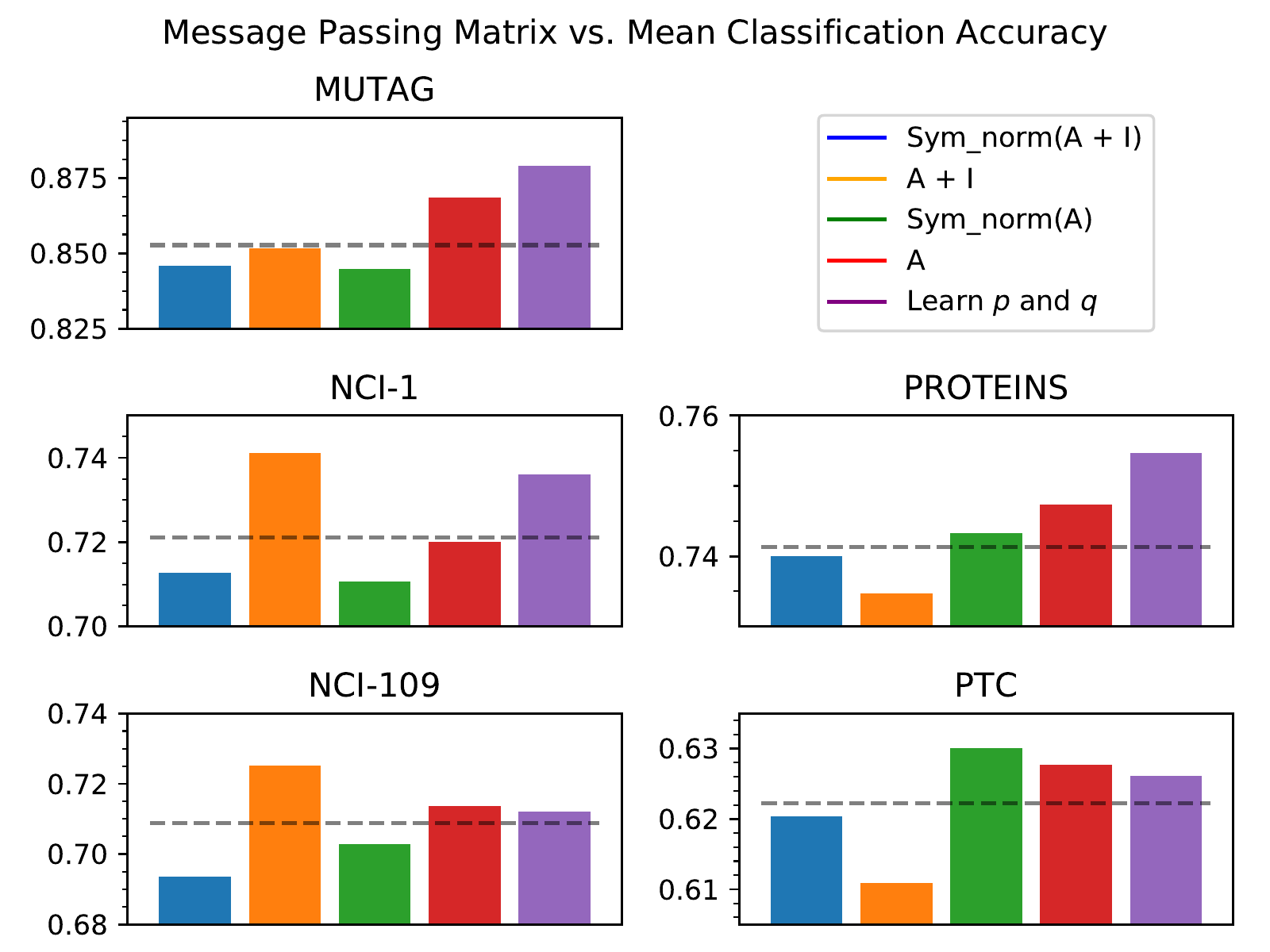}
  \caption{Each plot shows the mean classification accuracy for each message passing matrix of our search space, alongside the accuracy of PiNet when $p$ and $q$ are learned during training. The dashed lines indicate the mean accuracy of the manual search.}
  \label{fig:all_matrices}
\end{figure}

In \autoref{fig:all_matrices} we observe that the optimal message passing matrix (when $p$ and $q$ are fixed for all layers, and $(p, q) \in \{0, 1\} \times \{0, 1\}$) varies depending on the particular set of graphs given. With $p$ and $q$ as trainable parameters ($(p, q) \in [0, 1] \times [0, 1]$, and $(p, q)$ may be different for each layer), we see that for the MUTAG and PROTEINS datasets the model learns values that outperform those found with our manual search. For the others, however, the model is unable to find the optimal $p$s and $q$s, suggesting that the model finds only local minima. We note however, that in every case tested, the model is able to learn the values $p$ and $q$ that give better than average classification performance when compared with our manual search space.

\subsection{Comparison Against Existing Methods}

As shown in \autoref{tbl:benchresults}, PiNet achieves competitive classification performance on all datasets tested. An independent two-sample $t$-test indicates PiNet achieves a statistically significant ($p$-value $< 0.05$) gain in only a few cases ($^*$); however, with competitive performance on all datasets it demonstrates a robustness to the properties of the different sets of graphs.

\section{Conclusion}

We have proposed PiNet, an end-to-end deep neural network graph classifier invariant to permutations of nodes in the input graphs. We have provided theoretical proof of its invariance to permutation, and demonstrated the utility in such a property empirically with a graph isomorphism classification task against a set of existing graph classifiers, achieving a statistically significant gain in classification accuracy on a range of small training set sizes. The permutation invariance is achieved through a differentiable attention mechanism in which the model learns the weight by which the states associated with each node should be aggregated into the final graph representation.

We have demonstrated that PiNet is able to learn an effective parametrisation of a message passing matrix that enables it to adapt to different types of graphs with a flexible state propagation and diffusion mechanism. Finally, we have shown PiNet's robustness to the properties of different sets of graphs in achieving consistently competitive classification performance against a set of existing techniques on five commonly used molecule datasets.

For future work we plan to explore more advanced aggregation mechanisms by which the latent representations learned for each node of the graph may be combined.

\section*{Acknowledgements}
We thank Braintree Ltd. for providing the full funding for this work.

\appendix
\cleardoublepage
\bibliographystyle{named}
\bibliography{IJCAI-Graph-Classifier-2019.bib}
\end{document}